\documentclass[twoside,11pt]{article}

\usepackage{mine}

\firstpageno{1}

\begin{document}

\title{On- and Off-Policy Monotonic Policy Improvement}

\author{\name 
          Ryo Iwaki 
        \email 
          ryo.iwaki@ams.eng.osaka-u.ac.jp \\
        \addr 
          Department of Adaptive Machine Systems\\
          Graduate School of Enginieering\\
          Osaka University\\
          2-1, Yamadaoka, Suita city, Osaka, Japan
        \AND
        \name
          Minoru Asada
        \email
          asada@ams.eng.osaka-u.ac.jp \\
        \addr 
          Department of Adaptive Machine Systems\\
          Graduate School of Enginieering\\
          Osaka University,\\
          2-1, Yamadaoka, Suita city, Osaka, Japan}

\maketitle

\begin{abstract}
  Monotonic policy improvement and off-policy learning are two main desirable properties 
  for reinforcement learning algorithms.
  In this paper, by lower bounding the performance difference of two policies,
  we show that the monotonic policy improvement is guaranteed 
  from on- and off-policy mixture samples.
  An optimization procedure which applies the proposed bound can be regarded as an off-policy natural policy gradient method.
  In order to support the theoretical result, 
  we provide a trust region policy optimization method using experience replay
  as a naive application of our bound,
  and evaluate its performance in two classical benchmark problems.
\end{abstract}

\section{Introduction}

Reinforcement learning (RL) aims to optimize the 
behavior of an agent which interacts sequentially with an unknown environment 
in order to maximize the long term future reward.
There are two main desirable properties for RL algorithms:
{\itshape monotonic policy improvement} and {\itshape off-policy learning}.

If the model of environment is available and the state is fully observable,
a sequence of greedy policies generated by {\itshape policy iteration} scheme are guaranteed to improve monotonically.
However, in the {\itshape approximate policy iteration} including RL,
the generated policy could perform worse and lead to policy oscillation or policy degradation 
\citep{Bertsekas2011,Wagner2011,Wagner2014}.
In order to avoid such phenomena, 
there are some efforts to guarantee the monotonic policy improvement
\citep{Kakade2002_CPI,Pirotta2013_SPI,Schulman2015,Thomas2015_HCPI,Abbasi2016}.

On the other hand, the use of off-policy data is also very crucial for real world applications.
In an off-policy setting, a policy which generate the data is different from a policy to be optimized.
There are many theoretical efforts to efficiently use off-policy data 
\citep{Precup2000,Maei2011_PhD,Degris2012_Off,Zhao2013,Silver2014,Thomas2015_HCOPE,Harutyunyan2016,Munos2016}.
Off-policy learning methods enable the agent to, for example,
optimize huge function approximators effectively \citep{Mnih2015,Wang2017_ICLR}
and learn a complex policy for humanoid robot control in real world by reusing very few data \citep{Sugimoto2016}.

The goal of this paper is to show that the monotonic policy improvement is guaranteed in the on- and off-policy mixture setting.
Extending the approach by \cite{Pirotta2013_SPI},
we derive a general performance bound for on- and off-policy mixture samples.
An optimization procedure which applies the proposed bound can be regarded as an off-policy natural policy gradient method.
In order to support the theoretical result, 
we provide a trust region policy optimization (TRPO) method \citep{Schulman2015} using experience replay \citep{Lin1992}
as a naive application of our bound,
and evaluate its performance in two classical benchmark problems.

\section{Preliminaries}

We consider an infinite horizon discounted Markov decision process (MDP).
An MDP is specified by a tuple $(\mathcal{S},\mathcal{A},\mathcal{P},\mathcal{R},\rho_0,\gamma)$.
$\mathcal{S}$ is a finite set of possible states of an environment
and $\mathcal{A}$ is a finite set of possible actions which an agent can choose.
$\mathcal{P}: \mathcal{S} \times \mathcal{A} \times \mathcal{S} \to \mathbb{R}$ is a Markovian state transition probability distribution,
$\mathcal{R}: \mathcal{S} \times \mathcal{A} \to \mathbb{R}$ is a bounded reward function,
$\rho_{0}: \mathcal{S} \to \mathbb{R}$ is a initial state distribution,
and $\gamma \in (0, 1)$ is a discount factor.
We are interested in the model-free RL, 
thus we suppose that $\mathcal{P}$ and $\mathcal{R}$ are unknown.

Let $\pi$ be a {\itshape policy} of the agent;
if the policy is deterministic, $\pi$ denotes the mapping between the state and action spaces,
$\pi: \mathcal{S} \to \mathcal{A}$,
and if the policy is stochastic, $\pi$ denotes the distribution over the state-action pair,
$\pi: \mathcal{S} \times \mathcal{A} \to \mathbb{R}$.
For each policy $\pi$,
there exists an unnormalized $\gamma$-discounted future state distribution 
for the initial state distribution $\rho_{0}$,
$\rho^{\pi}(s) = \sum_{t=0}^{\infty} \gamma^{t} {\rm Pr}\,(s_{t}=s|\pi,\rho_{0})$.
We define the state value function $V^{\pi}(s)$,
the action value function $Q^{\pi}(s,a)$,
and the advantage function $A^{\pi}(s,a)$
for the policy $\pi$
as follows:
\begin{align*}
  V^{\pi}(s) 
  &= {\mathbb{E}} 
  \left[\sum_{t=0}^{\infty} \gamma^{t} \mathcal{R}(s_t,a_t) \,| s_{0}=s \right] 
  ,
  \\
  Q^{\pi}(s,a) 
  &= {\mathbb{E}} 
  \left[\sum_{t=0}^{\infty} \gamma^{t} \mathcal{R}(s_t,a_t) \,| s_{0}=s,a_{0}=a \right]
  ,
  \\
  A^{\pi}(s,a) &= Q^{\pi}(s,a) - V^{\pi}(s)
  .
\end{align*}
Note that the following Bellman equations hold:
\begin{align*}
  V^{\pi}(s)
   &=  \sum_{a\in\mathcal{A}} \pi(a|s)\left( \mathcal{R}(s,a)
    + \gamma \sum_{s'\in\mathcal{S}} \mathcal{P}(s'|s,a) V^{\pi}(s') \right)
  ,
  \\
  Q^{\pi}(s,a) 
   &= \mathcal{R}(s,a)
    + \gamma \sum_{s'\in\mathcal{S}} \mathcal{P}(s'|s,a) 
     \sum_{a'\in\mathcal{A}} \pi(a'|s') Q^{\pi}(s',a')
  .
\end{align*}
Furthermore, we define the advantage of a policy $\pi'$ over the policy $\pi$ for each state $s$:
\begin{align*}
  {\bar A}^{\pi}_{\pi'}(s)
  &= \sum_{a\in\mathcal{A}} \pi'(a|s) A^{\pi}(s,a)
  = \sum_{a\in\mathcal{A}} \left( \pi'(a|s) -  \pi(a|s) \right) Q^{\pi}(s,a)
  .
\end{align*}
The purpose of the agent is to find a policy $\pi^{*}$ which maximizes 
the expected discounted reward $\eta(\pi)$:
\begin{align*}
  \pi^{*} \in \underset{\pi}{\rm arg\,max}~ &\eta(\pi),
\end{align*}
where
\begin{align*}
  \eta(\pi)
   &= \sum_{s\in\mathcal{S}} \rho_{0} V^{\pi}(s)
   = \sum_{s\in\mathcal{S}} \rho^{\pi}(s) \sum_{a\in\mathcal{A}} \pi(a|s) \mathcal{R}(s,a).
\end{align*}

In the following, we use the matrix notation for the previous equations as in \cite{Pirotta2013_SPI}:
\begin{align}
  {\mathbf v}^{\pi}
   &= {\bm \Pi}^{\pi} \left({\mathbf r} + \gamma {\mathbf P} {\mathbf v}^{\pi}\right)
    = {\mathbf r}^{\pi} + \gamma {\mathbf P}^{\pi} {\mathbf v}^{\pi}
    = \left({\mathbf I} - \gamma {\mathbf P}^{\pi}\right)^{-1} {\mathbf r}^{\pi}
  ,\nonumber\\
  {\mathbf q}^{\pi}
   &= {\mathbf r} + \gamma {\mathbf P} {\bm \Pi}^{\pi} {\mathbf q}^{\pi}
    = {\mathbf r} + \gamma {\mathbf P} {\mathbf v}^{\pi}
  ,\nonumber\\
  \bar{\mathbf A}^{\pi}_{\pi'}
   &= {\bm \Pi}^{\pi'} {\mathbf A}^{\pi}
    = \left( {\bm \Pi}^{\pi'} - {\bm \Pi}^{\pi} \right) {\mathbf q}^{\pi}
  ,\nonumber\\
  \eta(\pi)
   &= {{\bm \rho}_{0}}^{\top} {\mathbf v}^{\pi}
    = {{\bm \rho}_{0}}^{\top} \left({\mathbf I} - \gamma {\mathbf P}^{\pi}\right)^{-1} {\mathbf r}^{\pi}
    = {{\bm \rho}^{\pi}}^{\top}{\mathbf r}^{\pi}
  ,\label{eq:vector_performance}
\end{align}
where
$\eta(\pi)$ is a scalar,
${\mathbf r}^{\pi}, {\mathbf v}^{\pi}, {\bm \rho}_{0}, {\bm \rho}^{\pi}$ 
and $\bar{\mathbf A}^{\pi}_{\pi'}$ are vectors of size $|\mathcal{S}|$,
${\mathbf r}, {\mathbf q}^{\pi}$ and ${\mathbf A}^{\pi}$ are vectors of size $|\mathcal{S}||\mathcal{A}|$,
${\mathbf P}$ is a stochastic matrix of size $|\mathcal{S}||\mathcal{A}|\times|\mathcal{S}|$ 
which contains the state transition probability distribution: ${\mathbf P} ((s,a),s') = \mathcal{P}(s'|s,a)$,
${\bm \Pi}^{\pi}$ is a stochastic matrix of size $|\mathcal{S}|\times|\mathcal{S}||\mathcal{A}|$ 
which contains the policy: ${\bm \Pi}^{\pi} (s,(s,a)) = \pi(a|s)$,
and ${\mathbf P}^{\pi} = {\bm \Pi}^{\pi} {\mathbf P}$ is a stochastic matrix of size $|\mathcal{S}|\times|\mathcal{S}|$
which represents the state transition matrix under the policy $\pi$.
Let $M$ be a matrix whose entries are $m_{ij}$,
then $\|M\|_{1} = \max_{j} \sum_{i}|m_{ij}|$, $\|M\|_{\infty} = \max_{i} \sum_{j}|m_{ij}|$ and $\|M\|_{1} = \|M^\top\|_{\infty}$.

\section{On- and Off-Policy Monotonic Policy Improvement Guarantee}

In this section, we show that the monotonic policy improvement is guaranteed 
from on- and off-policy mixture samples.

First, we introduce two lemmas to provide the main theorem.
The first lemma states that
the difference between the performances of any policies $\pi$ and $\pi'$ is given as a function of the advantage.
\begin{lemma}\label{lm:difference_of_performances_as_advantage}
  \citep[lemma 6.1]{Kakade2002_CPI}
  Let $\pi$ and $\pi'$ be any stationary policies. Then:
  \begin{align}
    \eta(\pi') - \eta(\pi) = 
    {{\bm \rho}^{\pi'}}^{\top} \bar{\mathbf A}^{\pi}_{\pi'}
    .
    \nonumber
  \end{align}
\end{lemma}
The second lemma gives a bound on the inner product of two vectors.
\begin{lemma}\label{lm:bound_of_inner_product}
  \citep[Corollary 2.4]{Haviv1984}
  Let ${\mathbf e}$ be a column vector of all entries are one. 
  For a vector ${\mathbf x}$ such that ${\mathbf x}^{\top}{\mathbf e}=0$
  and any vector ${\mathbf y}$,
  it holds that
  \begin{align}
    |{\mathbf x}^{\top}{\mathbf y}|
     \leq \|{\mathbf x}\|_{1} \frac{\max_{i,j}|{\mathbf y}_{i} - {\mathbf y}_{j}|}{2}
    .
    \nonumber
  \end{align}
\end{lemma}
Next we provide a bound to the difference between the $\gamma$-discounted state distributions 
for any stationary policies $\pi, \pi'$ and $\beta$.
\begin{lemma}\label{lm:bound_of_mixture_state_distributions}
  Let $\pi, \pi'$ and $\beta$ be any stationary policies 
  for an infinite-horizon MDP with state transition probability ${\mathbf P}$.
  Let $\alpha \in [0,1]$ be a mixture coefficient of on- and off-policy samples.
  Then the $L_{1}$-norm of the difference between the $\gamma$-discounted state distributions
  is upper bounded as follows:
  \begin{align}
    &\left\|{\bm \rho}^{\pi'} - (\alpha {\bm \rho}^{\pi} + (1-\alpha){\bm \rho}^{\beta})\right\|_{1}
    \nonumber\\
    &\qquad
    \leq
    \frac{\gamma}{1-\gamma}
    \left(
      \alpha\left\|{\mathbf P}^{\pi'} - {\mathbf P}^{\pi}\right\|_{\infty}
       + (1-\alpha) \left\|{\mathbf P}^{\pi'} - {\mathbf P}^{\beta}\right\|_{\infty}
    \right)
    \left\|\left(
      {\mathbf I}-\gamma{\mathbf P}^{\pi'}
    \right)^{-1}\right\|_{\infty}
    .
    \nonumber
  \end{align}
\end{lemma}
\begin{proof}
  Eq. (\ref{eq:vector_performance}) indicates that for any policy $\pi$ and any initial state distribution $\rho_{0}$,
  $\gamma$-discounted state distribution ${{\bm \rho}^{\pi}}$ satisfies
  \begin{align}
    {{\bm \rho}^{\pi}}^{\top} = {{\bm \rho}_{0}}^{\top} + \gamma {{\bm \rho}^{\pi}}^{\top} {\mathbf P}^{\pi}
    .
    \nonumber
  \end{align}
  It follows that
  \begin{align}
    &\left({{\bm \rho}^{\pi'}} - 
    \left( \alpha {{\bm \rho}^{\pi}} + (1-\alpha) {{\bm \rho}^{\beta}} \right) \right)^{\top}
    \nonumber\\
    &= \gamma {{\bm \rho}^{\pi'}}^{\top} {\mathbf P}^{\pi'}
    - 
    \left( 
      \alpha \gamma {{\bm \rho}^{\pi}}^{\top} {\mathbf P}^{\pi}
      + (1-\alpha) \gamma {{\bm \rho}^{\beta}}^{\top} {\mathbf P}^{\beta}
    \right)
    \nonumber\\
    &= \gamma 
      \left({{\bm \rho}^{\pi'}} - 
      \left( \alpha {{\bm \rho}^{\pi}} + (1-\alpha) {{\bm \rho}^{\beta}} \right) \right)^{\top}
      {\mathbf P}^{\pi'}
    + 
    \gamma \left( 
    \alpha {{\bm \rho}^{\pi}}^{\top} \left( {\mathbf P}^{\pi'} - {\mathbf P}^{\pi} \right)
    + (1-\alpha) {{\bm \rho}^{\beta}}^{\top} \left( {\mathbf P}^{\pi'} - {\mathbf P}^{\beta}\right)
    \right)
    \nonumber\\
    &= \gamma \left( 
    \alpha {{\bm \rho}^{\pi}}^{\top} \left( {\mathbf P}^{\pi'} - {\mathbf P}^{\pi} \right)
    + (1-\alpha) {{\bm \rho}^{\beta}}^{\top} \left( {\mathbf P}^{\pi'} - {\mathbf P}^{\beta}\right)
    \right)
    \sum_{t=0}^{\infty} \left( \gamma {\mathbf P}^{\pi'} \right)^{t}
    \label{eq:proof:lm:bound_of_mixture_state_distributions:Neumann_series}
    \\
    &= \gamma \left( 
    \alpha {{\bm \rho}^{\pi}}^{\top} \left( {\mathbf P}^{\pi'} - {\mathbf P}^{\pi} \right)
    + (1-\alpha) {{\bm \rho}^{\beta}}^{\top} \left( {\mathbf P}^{\pi'} - {\mathbf P}^{\beta}\right)
    \right) 
    \left( {\mathbf I}-\gamma{\mathbf P}^{\pi'} \right)^{-1}
    . 
    \label{eq:proof:lm:bound_of_mixture_state_distributions:Neumann_series_converges}
  \end{align}
  The equality (\ref{eq:proof:lm:bound_of_mixture_state_distributions:Neumann_series}) follows from the successive substitution.
  Since ${\mathbf P}^{\pi'}$ is a stochastic matrix, 
  the inverse of ${\mathbf I}-\gamma{\mathbf P}^{\pi'}$ exsits for any $\gamma < 1$,
  thus Neumann series converges 
  and (\ref{eq:proof:lm:bound_of_mixture_state_distributions:Neumann_series_converges}) follows.
  Therefore, it follows that
  \begin{align*}
    &\left\|{{\bm \rho}^{\pi'}} - 
      \left( \alpha {{\bm \rho}^{\pi}} + (1-\alpha) {{\bm \rho}^{\beta}} \right) \right\|_{1}
    \nonumber\\
    &=\left\|\left({{\bm \rho}^{\pi'}} - 
      \left( \alpha {{\bm \rho}^{\pi}} + (1-\alpha) {{\bm \rho}^{\beta}} \right) \right)^{\top}\right\|_{\infty}
    \nonumber\\
    &= \gamma \left\|
    \left( 
    \alpha {{\bm \rho}^{\pi}}^{\top} \left( {\mathbf P}^{\pi'} - {\mathbf P}^{\pi} \right)
    + (1-\alpha) {{\bm \rho}^{\beta}}^{\top} \left( {\mathbf P}^{\pi'} - {\mathbf P}^{\beta}\right)
    \right) 
    \left( {\mathbf I}-\gamma{\mathbf P}^{\pi'} \right)^{-1}
    \right\|_{\infty}
    \nonumber\\
    &\leq \gamma 
    \left( 
    \alpha \left\|{{\bm \rho}^{\pi}}^{\top}\right\|_{\infty} \left\| {\mathbf P}^{\pi'} - {\mathbf P}^{\pi} \right\|_{\infty}
    + (1-\alpha) \left\|{{\bm \rho}^{\beta}}^{\top}\right\|_{\infty} \left\| {\mathbf P}^{\pi'} - {\mathbf P}^{\beta}\right\|_{\infty}
    \right) 
    \left\|\left( {\mathbf I}-\gamma{\mathbf P}^{\pi'} \right)^{-1}
    \right\|_{\infty}
    \nonumber\\
    &\leq 
    \frac{\gamma}{1-\gamma}
    \left(
      \alpha\left\|{\mathbf P}^{\pi'} - {\mathbf P}^{\pi}\right\|_{\infty}
       + (1-\alpha) \left\|{\mathbf P}^{\pi'} - {\mathbf P}^{\beta}\right\|_{\infty}
    \right)
    \left\|\left(
      {\mathbf I}-\gamma{\mathbf P}^{\pi'}
    \right)^{-1}\right\|_{\infty}
    .
  \end{align*}
\end{proof}
The following corollary 
gives a looser but model-free bound.
\begin{corollary}\label{co:model_free_bound_of_mixture_state_distributions}
  Let $\pi, \pi'$ and $\beta$ be any stationary policies.
  Let $\alpha \in [0,1]$ be a mixture coefficient of on- and off-policy samples.
  Then the $L_{1}$-norm of the difference between the $\gamma$-discounted state distributions
  is upper bounded as follows:
  \begin{align}
    &\left\|{\bm \rho}^{\pi'} - \left(\alpha {\bm \rho}^{\pi} + (1-\alpha){\bm \rho}^{\beta}\right)\right\|_{1}
    \leq
    \frac{\gamma}{(1-\gamma)^{2}}
    \left(
      \alpha\left\|{\bm \Pi}^{\pi'} - {\bm \Pi}^{\pi}\right\|_{\infty}
       + (1-\alpha) \left\|{\bm \Pi}^{\pi'} - {\bm \Pi}^{\beta}\right\|_{\infty}
    \right)
    .
    \nonumber
  \end{align}
\end{corollary}
\begin{proof}
  From Lemma \ref{lm:bound_of_mixture_state_distributions}, it follows that
  \begin{align*}
    &
    \left\|{\bm \rho}^{\pi'} - (\alpha {\bm \rho}^{\pi} + (1-\alpha){\bm \rho}^{\beta})\right\|_{1}
    \nonumber\\
    &\leq
    \frac{\gamma}{1-\gamma}
    \left(
      \alpha\left\|{\mathbf P}^{\pi'} - {\mathbf P}^{\pi}\right\|_{\infty}
       + (1-\alpha) \left\|{\mathbf P}^{\pi'} - {\mathbf P}^{\beta}\right\|_{\infty}
    \right)
    \left\|\left(
      {\mathbf I}-\gamma{\mathbf P}^{\pi'}
    \right)^{-1}\right\|_{\infty}
    \nonumber\\
    &\leq
    \frac{\gamma}{1-\gamma}
    \left(
      \alpha\left\|{\bm \Pi}^{\pi'} - {\bm \Pi}^{\pi}\right\|_{\infty}
       + (1-\alpha) \left\|{\bm \Pi}^{\pi'} - {\bm \Pi}^{\beta}\right\|_{\infty}
    \right)
    \left\|{\mathbf P}\right\|_{\infty}
    \sum_{t=0}^{\infty} \gamma \left\| {\mathbf P}^{\pi'} \right\|_{\infty}^{t}
    \nonumber\\
    &\leq
    \frac{\gamma}{(1-\gamma)^{2}}
    \left(
      \alpha\left\|{\bm \Pi}^{\pi'} - {\bm \Pi}^{\pi}\right\|_{\infty}
       + (1-\alpha) \left\|{\bm \Pi}^{\pi'} - {\bm \Pi}^{\beta}\right\|_{\infty}
    \right)
    .
  \end{align*}
\end{proof}
The main theorem is given by combining 
Lemma \ref{lm:difference_of_performances_as_advantage}, 
Lemma \ref{lm:bound_of_inner_product},
and Corollary \ref{co:model_free_bound_of_mixture_state_distributions}.
\begin{theorem}\label{tm:on_and_off_mono_policy_improvement_with_TV}
  (On- and Off-Policy Monotonic Policy Improvement Guarantee)
  Let $\pi$ and $\pi'$ be any stationary target policies and $\beta$ be any stationary behavior policy.
  Let $\alpha \in [0,1]$ be a mixture coefficient of on- and off-policy samples.
  Then the difference between the performances of $\pi'$ and $\pi$ is lower bounded as follows:
  \begin{align}
    &\eta(\pi') - \eta(\pi) 
    \geq \alpha{{\bm \rho}^{\pi}}^{\top} \bar{\mathbf A}^{\pi}_{\pi'}
    + (1-\alpha){{\bm \rho}^{\beta}}^{\top} \bar{\mathbf A}^{\pi}_{\pi'}
    \nonumber\\
    &\qquad
    - \frac{\gamma}{(1-\gamma)^{2}}
    \left(
      \alpha\left\|{\bm \Pi}^{\pi'} - {\bm \Pi}^{\pi}\right\|_{\infty}^{2}
       + (1-\alpha) \left\|{\bm \Pi}^{\pi'} - {\bm \Pi}^{\pi}\right\|_{\infty}
                        \left\|{\bm \Pi}^{\pi'} - {\bm \Pi}^{\beta}\right\|_{\infty}
    \right)
    \left\|{\mathbf q}^{\pi}\right\|_{\infty}
    .
    \nonumber
  \end{align}
\end{theorem}
\begin{proof}
  From Lemma \ref{lm:difference_of_performances_as_advantage}, it follows that
  \begin{align}
    \eta(\pi') - \eta(\pi)
    &= {{\bm \rho}^{\pi'}}^{\top} \bar{\mathbf A}^{\pi}_{\pi'}
    \nonumber\\
    &= {{\bm \rho}^{\pi'}}^{\top} \bar{\mathbf A}^{\pi}_{\pi'}
     + \alpha 
        \left(
        {{\bm \rho}^{\pi}}^{\top} \bar{\mathbf A}^{\pi}_{\pi'} 
        - {{\bm \rho}^{\pi}}^{\top} \bar{\mathbf A}^{\pi}_{\pi'}
        \right)
     + (1-\alpha)
        \left(
        {{\bm \rho}^{\beta}}^{\top} \bar{\mathbf A}^{\pi}_{\pi'} 
        - {{\bm \rho}^{\beta}}^{\top} \bar{\mathbf A}^{\pi}_{\pi'}
        \right)      
    \nonumber\\
    &= \alpha {{\bm \rho}^{\pi}}^{\top} \bar{\mathbf A}^{\pi}_{\pi'}
       + (1-\alpha) {{\bm \rho}^{\beta}}^{\top} \bar{\mathbf A}^{\pi}_{\pi'}
     - \left({\bm \rho}^{\pi'} - \left(\alpha {{\bm \rho}^{\pi}} + (1-\alpha) {{\bm \rho}^{\beta}} \right) \right)^{\top} 
      \bar{\mathbf A}^{\pi}_{\pi'}
    ,
    \nonumber
  \end{align}
  where $\alpha \in [0,1]$.
  Note that for any policy $\pi$, 
  the $\gamma$-discounted state distribution 
  ${\bm \rho}^{\pi}$ satisfies ${{\bm \rho}^{\pi}}^{\top} {\mathbf e}\!=\!1$,
  thus 
  $({\bm \rho}^{\pi'} \!-\! (\alpha {{\bm \rho}^{\pi}} \!+\! (1\!-\!\alpha) {{\bm \rho}^{\beta}} ) )^{\top} {\mathbf e}\!=\!0$.
  Therefore from Lemma \ref{lm:bound_of_inner_product}
  and Corollary \ref{co:model_free_bound_of_mixture_state_distributions},
  it follows that
  \begin{align}
    &\eta(\pi') - \eta(\pi)
    \nonumber\\
    &= \alpha {{\bm \rho}^{\pi}}^{\top} \bar{\mathbf A}^{\pi}_{\pi'}
       + (1-\alpha) {{\bm \rho}^{\beta}}^{\top} \bar{\mathbf A}^{\pi}_{\pi'}
     - \left({\bm \rho}^{\pi'} - \left(\alpha {{\bm \rho}^{\pi}} + (1-\alpha) {{\bm \rho}^{\beta}} \right) \right)^{\top} 
      \bar{\mathbf A}^{\pi}_{\pi'}
    \nonumber\\
    &\geq \alpha {{\bm \rho}^{\pi}}^{\top} \bar{\mathbf A}^{\pi}_{\pi'}
       + (1-\alpha) {{\bm \rho}^{\beta}}^{\top} \bar{\mathbf A}^{\pi}_{\pi'}
     - \left\|{\bm \rho}^{\pi'} - \left(\alpha {{\bm \rho}^{\pi}} + (1-\alpha) {{\bm \rho}^{\beta}} \right) \right\|_{1} 
      \frac{\epsilon}{2}
    \nonumber\\
    &
    \geq \alpha{{\bm \rho}^{\pi}}^{\top} \bar{\mathbf A}^{\pi}_{\pi'}
    + (1-\alpha){{\bm \rho}^{\beta}}^{\top} \bar{\mathbf A}^{\pi}_{\pi'}
    - \frac{\gamma}{(1-\gamma)^{2}}
    \left(
      \alpha\left\|{\bm \Pi}^{\pi'} - {\bm \Pi}^{\pi}\right\|_{\infty}
       + (1-\alpha) \left\|{\bm \Pi}^{\pi'} - {\bm \Pi}^{\beta}\right\|_{\infty}
    \right)
    \frac{\epsilon}{2}
    ,
    \nonumber
  \end{align}
  where $\epsilon = \max_{s,s'}|\bar{\mathbf A}^{\pi}_{\pi'}(s) - \bar{\mathbf A}^{\pi}_{\pi'}(s')|$.
  The theorem follows by upper bounding $\epsilon / 2$:
  \begin{align*}
    \frac{\epsilon}{2} 
    &\leq \left\|\bar{\mathbf A}^{\pi}_{\pi'}\right\|_{\infty}
    = \left\|\left({\bm \Pi}^{\pi'} - {\bm \Pi}^{\pi}\right){\mathbf q}^{\pi}\right\|_{\infty}
    \leq \left\|{\bm \Pi}^{\pi'} - {\bm \Pi}^{\pi}\right\|_{\infty}
          \left\|{\mathbf q}^{\pi}\right\|_{\infty}
    .
  \end{align*}
\end{proof}
Note that for any {\itshape stochastic} policies,
\begin{align*}
  \left\|{\bm \Pi}^{\pi'} - {\bm \Pi}^{\pi}\right\|_{\infty}
  = \underset{s\in\mathcal{S}}{\max}
  \sum_{a\in\mathcal{A}}\left|\pi'(a|s)-\pi(a|s)\right|
  = \underset{s\in\mathcal{S}}{\max}
  \sum_{a\in\mathcal{A}}\left|\pi(a|s)-\pi'(a|s)\right|
\end{align*}
is identical to the maximum total variation distance
between the policies with respect to the state
\footnote{
  $D_{\rm TV}(\pi\|\pi') = \frac{1}{2}\sum_{a\in\mathcal{A}}\left|\pi(a|s)-\pi'(a|s)\right|$ 
  is another common definition of the total variation distance.
},
$D_{\rm TV}^{\rm max}(\pi\|\pi') = {\rm max}_{s\in\mathcal{S}} D_{\rm TV}(\pi(\cdot|s)\|\pi'(\cdot|s))$.
Thus Pinsker's inequality, 
\begin{align*}
  \frac{1}{2}D_{\rm TV} (\pi \| \pi')^2 \leq D_{\rm KL} (\pi \| \pi')
  , 
\end{align*}
where $D_{\rm KL} (\pi \| \pi')$ is the Kullback-Leibler divergence between two policies,
yields following corollary.
\begin{corollary}\label{co:on_and_off_mono_policy_improvement_with_KL}
  Let $\pi$ and $\pi'$ be any stochastic stationary target policies and $\beta$ be any stochastic stationary behavior policy.
  For $\alpha \in [0,1]$,
  the difference between the performances of $\pi'$ and $\pi$ is lower bounded as follows:
  \begin{align}
    &\eta(\pi') - \eta(\pi) 
    \geq 
    \alpha
    {{\bm \rho}^{\pi}}^{\top} \bar{\mathbf A}^{\pi}_{\pi'}
    + (1-\alpha)
    {{\bm \rho}^{\beta}}^{\top} \bar{\mathbf A}^{\pi}_{\pi'}
    \nonumber\\
    &\qquad- \frac{2\gamma}{(1-\gamma)^{2}}
    \left(
      \alpha
      D_{\rm KL}^{\rm max} (\pi \| \pi') 
       + (1-\alpha) \left(D_{\rm KL}^{\rm max} (\pi \| \pi') D_{\rm KL}^{\rm max} (\beta \| \pi') \right)^{1/2}
    \right)
    \left\|{\mathbf q}^{\pi}\right\|_{\infty}
    .
    \label{eq:on_and_off_mono_policy_improvement_with_KL}
  \end{align}
\end{corollary}
\begin{remark}\label{co:KL_for_beta_can_be_ignored}
  Corollary \ref{co:on_and_off_mono_policy_improvement_with_KL} states that
  the penalty to the policy improvement is governed by 
  $D_{\rm KL}^{\rm max} (\pi \| \pi'), D_{\rm KL}^{\rm max} (\beta \| \pi')$ and $\alpha$.
  $D_{\rm KL}^{\rm max} (\beta \| \pi')$ indicates the `off-policy-ness'.
  In the penalty term, $D_{\rm KL}^{\rm max} (\beta \| \pi')$ is multiplied 
  by $D_{\rm KL}^{\rm max} (\pi \| \pi')$ and $1-\alpha$,
  thus, the monotonic policy improvement could be established with sufficiently small $D_{\rm KL}^{\rm max} (\pi \| \pi')$ 
  and appropriate value of $\alpha$.
\end{remark}
In order to improve the policy monotonically, we should choose the policy $\pi'$ 
with which the right hand side of (\ref{eq:on_and_off_mono_policy_improvement_with_KL}) is positive.
However, as discussed by \cite{Schulman2015}, 
evaluating $D_{\rm KL}^{\rm max} (\pi \| \pi')$ or $D_{\rm KL}^{\rm max} (\beta \| \pi')$ is intractable in general 
because it requires to calculate KL divergence at every point in the state space.
Instead, we consider the following expected KL divergence: 
\begin{align}
  &\alpha \mathbb{E}_{s\sim\rho^\pi} \left[ D_{\rm KL} (\pi \| \pi') \right]
  + (1\!-\!\alpha) \left( \mathbb{E}_{s\sim\rho^\pi} \left[ D_{\rm KL} (\pi \| \pi') \right]
              \mathbb{E}_{s\sim\rho^{\beta}} \left[ D_{\rm KL} (\beta \| \pi') \right]\right)^{1/2}
  .
  \label{eq:mixture_metric}
\end{align}
Note that the metric,
\begin{align}
  \mathbb{E}_{s\sim\rho^\pi} \left[ D_{\rm KL} (\pi \| \pi') \right]
  ,
  \label{eq:on_policy_metric}
\end{align}
is identical to the one used in the literature of the natural policy gradient
\citep{Kakade2001_NPG,Bagnell2003,Peters2003,Morimura2005_NTD_E}.
Analogously, an optimization procedure 
which uses the metric (\ref{eq:mixture_metric}) and
applies the bound (\ref{eq:on_and_off_mono_policy_improvement_with_KL}) approximately 
can be regarded as a variant of off-policy natural policy gradient method.

\section{Experiment}

In order to support our theoretical result, 
we evaluate the naive application of our bound in two classic benchmark problems.
Note that the method presented here is just one possible implementation 
to perform monotonic policy improvement approximately from on- and off-policy mixture samples.

\subsection{TRPO with Experience Replay}

First, we propose to directly use the experience replay \citep{Lin1992} 
in the trust region policy optimization scheme \citep{Schulman2015}.

Analogous to the argument in the Remark \ref{co:KL_for_beta_can_be_ignored},
in the mixture metric \ref{eq:mixture_metric},
$\mathbb{E}_{s\sim\rho^{\beta}} \left[ D_{\rm KL} (\beta \| \pi') \right]$ is multiplied by 
$\mathbb{E}_{s\sim\rho^\pi} \left[ D_{\rm KL} (\pi \| \pi') \right]$
and $1-\alpha$.
Thus, here we simply use the on-policy metric (\ref{eq:on_policy_metric}) as a constraint,
and investigate whether the monotonic policy improvement could be established 
with small $\mathbb{E}_{s\sim\rho^\pi} \left[ D_{\rm KL} (\pi \| \pi') \right]$ 
and large value of $\alpha$.

Suppose that we would like to optimize the policy $\pi_{\theta}$ with parameter $\theta$.
Then as done in TRPO \citep{Schulman2015},
the constrained optimization problem we should solve to update $\theta$ is:
\begin{align}
  &\underset{\theta'}{\rm maximize} \quad 
  L(\theta',\theta,\beta,\alpha)
  \nonumber\\
  &\qquad
  =
  \alpha \mathbb{E}_{s\sim\rho^{\pi_{\theta}}, a\sim\pi_{\theta}} 
  \left[ \frac{\pi_{\theta'}(a|s)}{\pi_{\theta}(a|s)} A^{\pi_{\theta}}(a|s) \right]
  + (1\!-\!\alpha) \mathbb{E}_{s\sim\rho^{\beta}, a\sim\beta} 
  \left[ \frac{\pi_{\theta'}(a|s)}{\beta(a|s)} A^{\pi_{\theta}}(a|s) \right]
  ,
  \label{eq:TRPO_with_ER:objective}
  \\
  &{\rm subject~to} \quad \mathbb{E}_{s\sim\rho^{\pi_{\theta}}} 
    \left[ D_{\rm KL} (\pi_{\theta}(\cdot|s) \| \pi_{\theta'}(\cdot|s)) \right] \leq \delta
  .
  \label{eq:TRPO_with_ER:on-policy_constraint}
\end{align}
By setting $\alpha = 1$, proposed method reduces to TRPO.
Note that $\alpha$ and $\beta$ can be varied at each update.
As the sampling from $\rho^{\beta}$ and $\beta$, we propose to use experience replay \citep{Lin1992}.
The optimization procedure in each training epoch is as follows:
\begin{enumerate}
  \item perform rollout with the policy $\pi_{\theta}$ and obtain on-policy trajectory,
  \item append on-policy trajectory to the replay buffer,
  \item draw off-policy trajectories from the replay buffer,
  \item solve the constrained optimization problem 
        (\ref{eq:TRPO_with_ER:objective}, \ref{eq:TRPO_with_ER:on-policy_constraint}) 
        using the on- and off-policy trajectories
        to generate new policy $\pi_{\theta'}$.
\end{enumerate}

\subsection{Experiment on Open AI Gym}

The experiment is conducted on the {\itshape Open AI Gym}.
The tasks are {\itshape Acrobot-v1} and {\itshape Pendulumn-v0}.
The agent is implemented based on the TRPO in the {\itshape baselines},
and modified to deal with off-policy samples.
Both the policy and the state value are approximated by feedforward neural network with two hidden layers, 
both of which consist of $32$ tanh units.
The surrogate objective (\ref{eq:TRPO_with_ER:objective}) is approximated by the generalized advantage estimation \citep{Schulman2016}.
The state value function is updated by using the $\lambda$-return as the target with on-policy trajectory.
The decay rate is set to $\lambda = 0.98$,
and the discount factor is set to $\gamma = 0.99$.
The trust region is set to $\delta = 0.01$.
Single on-policy trajectory consists of $1000$ transitions.
Previous $100$ trajectories are stored in the replay buffer and $10$ trajectries are drawn as the off-policy samples.
The learning results for various value of mixture coefficient, $\alpha$, 
are shown in Figure \ref{fig:learning_result}.
Each learning result is the average of $10$ independent runs with different seeds.
For Acrobot-v1,
the on- and off-policy mixture learning with $\alpha = 0.5, 0.75, 0.8, 0.99$ outperformed the on-policy TRPO, $\alpha = 1.0$.
Furthermore, monotonic policy improvement was established from off-policy samples only, , $\alpha = 0.0$.
For Pendulumn-v0,
the on- and off-policy mixture samples with $\alpha \geq 0.9$ accelerated the learning 
and resulted in the faster learning than TRPO, $\alpha = 1.0$.
However, the smaller $\alpha$ becomes, the slower the learning progresses.
This empirical result emphasizes the importance to deal with the mixture metric (\ref{eq:mixture_metric}) as the constraint.
\begin{figure}[t]
  \subfigure[Acrobot-v1]{
    \includegraphics[width=0.5\columnwidth,clip]{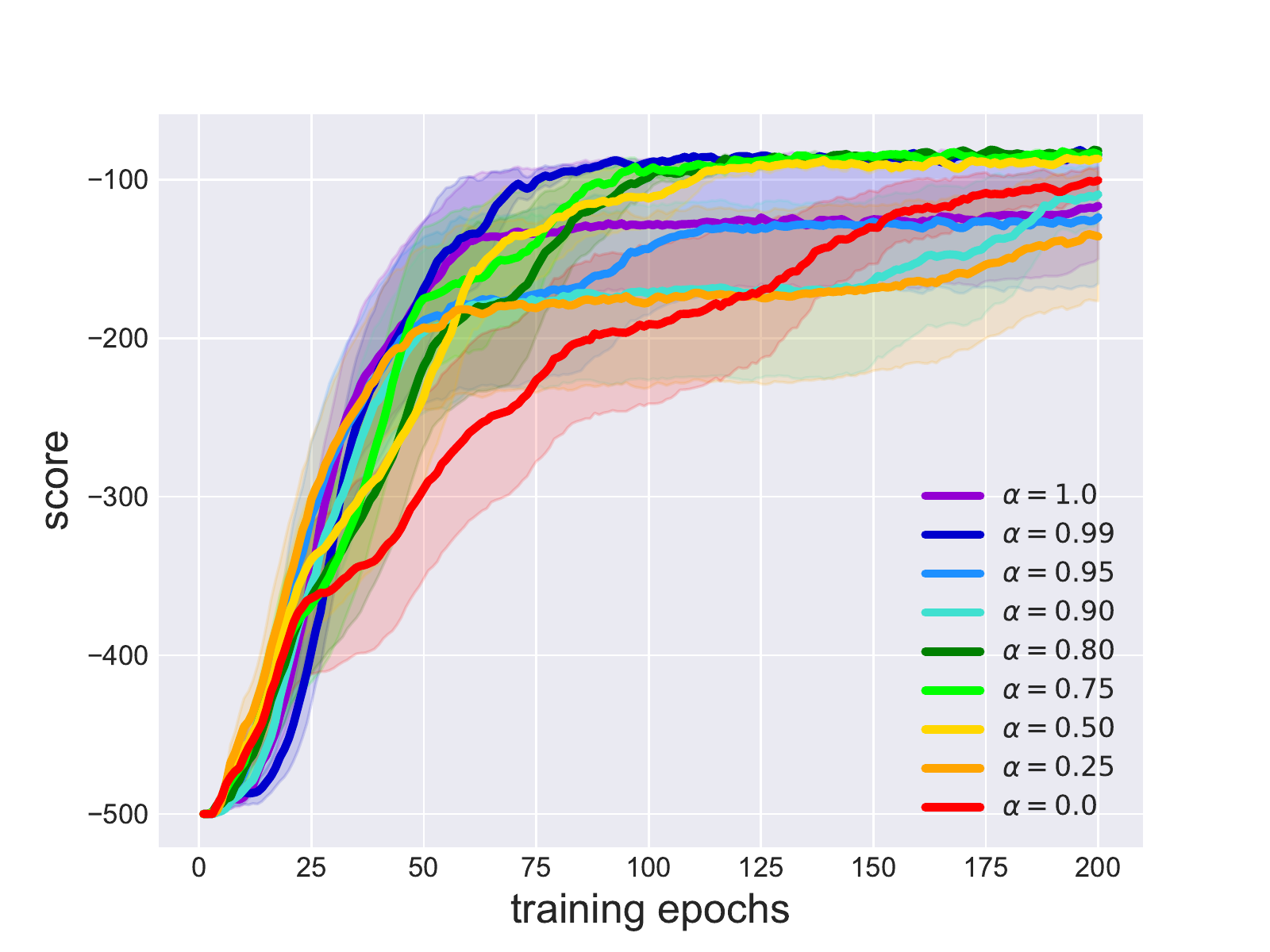}
  }
  \subfigure[Pendulumn-v0]{
    \includegraphics[width=0.5\columnwidth,clip]{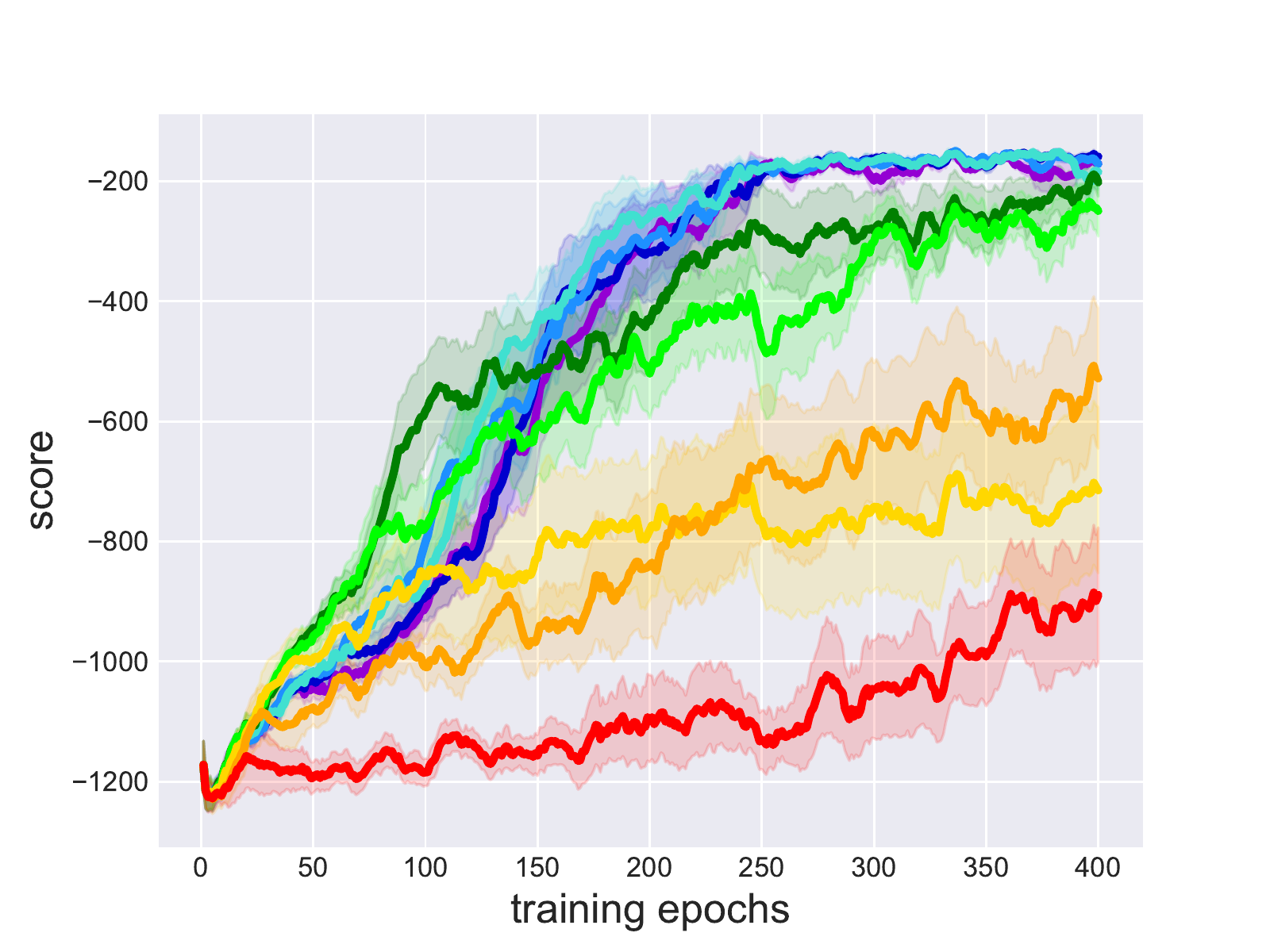}
  }
  \caption{Learning result in Open AI Gym. 
            $\alpha = 1.0$ corresponds to TRPO.
            Shaded areas denote standard errors.}
  \label{fig:learning_result}
\end{figure}

\section{Related Works}

The theoretical result and proposed method are extension of the works by \cite{Pirotta2013_SPI,Schulman2015}
to on- and off-policy mixture case.
\cite{Thomas2015_HCPI} proposed an algorithm with monotonic improvement which can use off-policy policy evaluation techique.
However, its computational complexity is high.
\cite{Gu2017_ICLR,Gu2017_NIPS} proposed to interpolate TRPO and deep deterministic policy gradient \citep{Lillicrap2016},
and \cite{Gu2017_NIPS} showed the performance bound for their on- and off-policy mixture update as well.
In our notation, the penalty terms in their performance bound has $D_\mathrm{KL}^{\rm max} (\pi \| \beta)$,
which is a constant with respect to the policy after update, $\pi'$.
In contrast, in Corollary \ref{co:on_and_off_mono_policy_improvement_with_KL},
our penalty terms has $D_\mathrm{KL}^{\rm max} (\beta \| \pi')$ instead of $D_\mathrm{KL}^{\rm max} (\pi \| \beta)$,
and which is multiplied by $D_\mathrm{KL}^{\rm max} (\pi \| \pi')$ and $1-\alpha$.
Furthermore, our bound is general in the sense that it does not specify how the policy is updatd.
Thus all the penalty terms in Eq. (\ref{eq:on_and_off_mono_policy_improvement_with_KL}) 
can be controlled in a well-designed optimization procedure to find $\pi'$.

\section{Conclusion and Outlook}

In this paper, we showed that the monotonic policy improvement is guaranteed 
from on- and off-policy mixture samples,
by lower bounding the performance difference of two policies.
An optimization scheme which applies the derived bound can be regarded as an off-policy natural policy gradient method.
In order to support the theoretical result, 
we provided the TRPO method using the experience replay as the naive application of our bound,
and evaluated the performance with various values of mixture coefficient.
An important direction is to find a practical algorithm which uses the metric (\ref{eq:mixture_metric}) as a constraint.
Determining $\alpha$ depending on 
$\mathbb{E}_{s\sim\rho^{\beta}} \left[ D_\mathrm{KL} (\beta \| \pi')\right]$
is also an interesting future work.


\vskip 0.2in
\bibliography{references}

\begin{thebibliography}{29}
\providecommand{\natexlab}[1]{#1}
\providecommand{\url}[1]{\texttt{#1}}
\expandafter\ifx\csname urlstyle\endcsname\relax
  \providecommand{\doi}[1]{doi: #1}\else
  \providecommand{\doi}{doi: \begingroup \urlstyle{rm}\Url}\fi

\bibitem[Abbasi-Yadkori et~al.(2016)Abbasi-Yadkori, Bartlett, and
  Wright]{Abbasi2016}
Yasin Abbasi-Yadkori, Peter~L. Bartlett, and Stephen~J. Wright.
\newblock A fast and reliable policy improvement algorithm.
\newblock In \emph{In Artificial Intelligence and Statistics}, pages
  1338--1346, 2016.

\bibitem[Bagnell and Schneider(2003)]{Bagnell2003}
J.~Andrew Bagnell and Jeff Schneider.
\newblock Covariant policy search.
\newblock In \emph{International Joint Conference on Artificial Intelligence},
  pages 1019--^^e2^^80^^931024, 2003.

\bibitem[Bertsekas(2011)]{Bertsekas2011}
Dimirti~P. Bertsekas.
\newblock Approximate policy iteration: A survey and some new methods.
\newblock \emph{Journal of Control Theory and Applications}, 9\penalty0 (3),
  2011.

\bibitem[Degris et~al.(2012)Degris, White, and Sutton]{Degris2012_Off}
Thomas Degris, Martha White, and Richard~S. Sutton.
\newblock Off-policy actor-critic.
\newblock In \emph{International Conference on Machine Learning}, 2012.

\bibitem[Gu et~al.(2017{\natexlab{a}})Gu, Lillicrap, Ghahramani, Turner, and
  Levine]{Gu2017_ICLR}
Shixiang Gu, Timothy Lillicrap, Zoubin Ghahramani, Richard~E Turner, and Sergey
  Levine.
\newblock Q-prop: Sample-efficient policy gradient with an off-policy critic.
\newblock In \emph{International Conference on Learning Representations},
  2017{\natexlab{a}}.

\bibitem[Gu et~al.(2017{\natexlab{b}})Gu, Lillicrap, Ghahramani, Turner,
  Sch{\"o}lkopf, and Levine]{Gu2017_NIPS}
Shixiang Gu, Timothy Lillicrap, Zoubin Ghahramani, Richard~E Turner, Bernhard
  Sch{\"o}lkopf, and Sergey Levine.
\newblock Interpolated policy gradient: Merging on-policy and off-policy
  gradient estimation for deep reinforcement learning.
\newblock In \emph{Advances in Neural Information Processing Systems},
  2017{\natexlab{b}}.

\bibitem[Harutyunyan et~al.(2016)Harutyunyan, Bellemare, Stepleton, and
  Munos]{Harutyunyan2016}
Anna Harutyunyan, Marc~G. Bellemare, Tom Stepleton, and R{\'e}mi Munos.
\newblock Q($\lambda$) with off-policy corrections.
\newblock In \emph{International Conference on Algorithmic Learning Theory},
  pages 305--320, 2016.

\bibitem[Haviv and Heyden(1984)]{Haviv1984}
Moshe Haviv and Ludo Van~Der Heyden.
\newblock Perturbation bounds for the stationary probabilities of a finite
  markov chain.
\newblock \emph{Advances in Applied Probability}, 16\penalty0 (4), 1984.
\newblock URL \url{http://www.jstor.org/stable/1427341}.

\bibitem[Kakade(2001)]{Kakade2001_NPG}
Sham Kakade.
\newblock A natural policy gradient.
\newblock In \emph{Advances in Neural Information Processing Systems},
  volume~14, 2001.

\bibitem[Kakade and Langford(2002)]{Kakade2002_CPI}
Sham Kakade and John Langford.
\newblock Approximately optimal approximate reinforcement learning.
\newblock In \emph{International Conference on Machine Learning}, volume~2,
  2002.

\bibitem[Lillicrap et~al.(2016)Lillicrap, Hunt, Pritzel, Heess, Erez, Tassa,
  Silver, and Wierstra]{Lillicrap2016}
Timothy~P. Lillicrap, Jonathan~J. Hunt, Alexander Pritzel, Nicolas Heess, Tom
  Erez, Yuval Tassa, David Silver, and Daan Wierstra.
\newblock Continuous control with deep reinforcement learning.
\newblock In \emph{International Conference on Learning Representations}, 2016.

\bibitem[Lin(1992)]{Lin1992}
Long-Ji Lin.
\newblock Self-improving reactive agents based on reinforcement learning,
  planning and teaching.
\newblock \emph{Machine Learning}, 8\penalty0 (3/4):\penalty0 69--97, 1992.

\bibitem[Maei(2011)]{Maei2011_PhD}
Hamid~Reza Maei.
\newblock \emph{Gradient Temporal-Difference Learning Algorithms}.
\newblock PhD thesis, University of Alberta, 2011.

\bibitem[Mnih et~al.(2015)Mnih, Kavukcuoglu, Silver, Rusu, Veness, Bellemare,
  Graves, Riedmiller, Fidjeland, Ostrovski, et~al.]{Mnih2015}
Volodymyr Mnih, Koray Kavukcuoglu, David Silver, Andrei~A Rusu, Joel Veness,
  Marc~G Bellemare, Alex Graves, Martin Riedmiller, Andreas~K Fidjeland, Georg
  Ostrovski, et~al.
\newblock Human-level control through deep reinforcement learning.
\newblock \emph{Nature}, 518\penalty0 (7540):\penalty0 529--533, 2015.

\bibitem[Morimura et~al.(2005)Morimura, Uchibe, and Doya]{Morimura2005_NTD_E}
Tetsuro Morimura, Eiji Uchibe, and Kenji Doya.
\newblock Utilizing natural gradient in temporal difference reinforcement
  learning with eligibility traces.
\newblock In \emph{International Symposium on Information Geometry and Its
  Applications}, pages 256--263, 2005.

\bibitem[Munos et~al.(2016)Munos, Stepleton, Harutyunyan, and
  Bellemare]{Munos2016}
R{\'e}mi Munos, Tom Stepleton, Anna Harutyunyan, and Marc~G. Bellemare.
\newblock Safe and efficient off-policy reinforcement learning.
\newblock In \emph{Advances in Neural Information Processing Systems}, 2016.

\bibitem[Peters et~al.(2003)Peters, Vijayakumar, and Schaal]{Peters2003}
Jan Peters, Sethu Vijayakumar, and Stefan Schaal.
\newblock Reinforcement learning for humanoid robotics.
\newblock In \emph{Third IEEE-RAS International Conference on Humanoid Robots},
  pages 1--20. American Association for Artificial Intelligence, 2003.

\bibitem[Pirotta et~al.(2013)Pirotta, Restelli, Pecorino, and
  Calandriello]{Pirotta2013_SPI}
Matteo Pirotta, Marcello Restelli, Alessio Pecorino, and Daniele Calandriello.
\newblock Safe policy iteration.
\newblock In \emph{International Conference on Machine Learning}, pages
  307--315, 2013.

\bibitem[Precup et~al.(2000)Precup, Sutton, and Singh]{Precup2000}
Doina Precup, Richard~S. Sutton, and Satinder Singh.
\newblock Eligibility traces for off-policy policy evaluation.
\newblock In \emph{International Conference on Machine Learning}, 2000.

\bibitem[Schulman et~al.(2015)Schulman, Levine, Moritz, Jordan, and
  Abbeel]{Schulman2015}
John Schulman, Sergey Levine, Philipp Moritz, Michael Jordan, and Pieter
  Abbeel.
\newblock Trust region policy optimization.
\newblock In \emph{International Conference on Machine Learning}, pages
  1889--1897, 2015.

\bibitem[Schulman et~al.(2016)Schulman, Moritz, Levine, Jordan, and
  Abbeel]{Schulman2016}
John Schulman, Philipp Moritz, Sergey Levine, Michael~I. Jordan, and Pieter
  Abbeel.
\newblock High-dimensional continuous control using generalized advantage
  estimation.
\newblock In \emph{International Conference on Learning Representations}, 2016.

\bibitem[Silver et~al.(2014)Silver, Lever, Heess, Degris, Wierstra, and
  Riedmiller]{Silver2014}
David Silver, Guy Lever, Nicolas Heess, Thomas Degris, Daan Wierstra, and
  Martin Riedmiller.
\newblock Deterministic policy gradient algorithms.
\newblock \emph{International Conference on Machine Learning}, pages 387--395,
  2014.

\bibitem[Sugimoto et~al.(2016)Sugimoto, Tangkaratt, Wensveen, Zhao, Sugiyama,
  and Morimoto]{Sugimoto2016}
Norikazu Sugimoto, Voot Tangkaratt, Thijs Wensveen, Tingting Zhao, Masashi
  Sugiyama, and Jun Morimoto.
\newblock Trial and error: Using previous experiences as simulation models in
  humanoid motor learning.
\newblock \emph{IEEE Robotics \& Automation Magazine}, 23\penalty0
  (1):\penalty0 96--105, 2016.

\bibitem[Thomas et~al.(2015{\natexlab{a}})Thomas, Theocharous, and
  Ghavamzadeh]{Thomas2015_HCOPE}
Philip~S. Thomas, Georgios Theocharous, and Mohammad Ghavamzadeh.
\newblock High confidence off-policy evaluation.
\newblock In \emph{AAAI}, 2015{\natexlab{a}}.

\bibitem[Thomas et~al.(2015{\natexlab{b}})Thomas, Theocharous, and
  Ghavamzadeh]{Thomas2015_HCPI}
Philip~S. Thomas, Georgios Theocharous, and Mohammad Ghavamzadeh.
\newblock High confidence policy improvement.
\newblock In \emph{International Conference on Machine Learning},
  2015{\natexlab{b}}.

\bibitem[Wagner(2011)]{Wagner2011}
Paul Wagner.
\newblock A reinterpretation of the policy oscillation phenomenon in
  approximate policy iteration.
\newblock In \emph{Advances in Neural Information Processing Systems}, 2011.

\bibitem[Wagner(2014)]{Wagner2014}
Paul Wagner.
\newblock Policy oscillation is overshooting.
\newblock \emph{Neural Networks}, 52:\penalty0 43--61, 2014.

\bibitem[Wang et~al.(2017)Wang, Bapst, Heess, Mnih, Munos, Kavukcuoglu, and
  de~Freitas]{Wang2017_ICLR}
Ziyu Wang, Victor Bapst, Nicolas Heess, Volodymyr Mnih, Remi Munos, Koray
  Kavukcuoglu, and Nando de~Freitas.
\newblock Sample efficient actor-critic with experience replay.
\newblock In \emph{International Conference on Learning Representations}, 2017.

\bibitem[Zhao et~al.(2013)Zhao, Hachiya, Tangkaratt, Morimoto, and
  Sugiyama]{Zhao2013}
Tingting Zhao, Hirotaka Hachiya, Voot Tangkaratt, Jun Morimoto, and Masashi
  Sugiyama.
\newblock Efficient sample reuse in policy gradients with parameter-based
  exploration.
\newblock \emph{Neural computation}, 25\penalty0 (6):\penalty0 1512--1547,
  2013.

\end{thebibliography}

\end{document}